\documentclass[a4paper, 11pt]{article}
\usepackage[utf8]{inputenc}
\usepackage[T1]{fontenc}
\usepackage[american]{babel}
\usepackage{csquotes}
\usepackage{geometry}
\usepackage{graphicx}
\usepackage[dvipsnames]{xcolor}
\usepackage{hyperref}
\usepackage{xurl}
\usepackage{enumitem}
\usepackage{amsmath, amstext, amssymb, amsopn, amsthm}
\usepackage{mathtools}  
\usepackage{stmaryrd}
\usepackage{bbm}
\usepackage{braket}
\usepackage{cancel}
\usepackage[capitalize]{cleveref}
\usepackage[backend=biber, style=numeric, sorting=nyt]{biblatex}
\usepackage{authblk}

\renewbibmacro{in:}{}
\DeclareFieldFormat{doi}{
  \texttt{doi}\addcolon\space\ifhyperref
  {\href{https://doi.org/#1}{\nolinkurl{#1}}}
  {\nolinkurl{#1}}
}
\DeclareFieldFormat{url}{
  \texttt{url}\addcolon\space\ifhyperref
  {\url{#1}}{\nolinkurl{#1}}
}

\hypersetup{
  colorlinks=true,
  linkcolor=Red,
  citecolor=Green,
  urlcolor=Turquoise
}
\newcommand{\N}{\mathbb{N}}

\newcommand{\R}{\mathbb{R}}
\renewcommand{\S}{\mathbb{S}}
\renewcommand{\P}{\mathbb{P}}
\newcommand{\E}{\mathbb{E}}
\newcommand{\F}{\mathcal{F}}
\renewcommand{\L}{\mathcal{L}}
\newcommand{\X}{\mathcal{X}}
\newcommand{\Y}{\mathcal{Y}}
\newcommand{\W}{\mathsf{W}}
\newcommand{\SW}{\mathsf{SW}}
\renewcommand{\d}{\mathsf{d}}
\newcommand{\s}{\mathsf{s}}

\DeclareMathOperator{\Lip}{Lip}

\DeclareMathOperator*{\supp}{supp}
\DeclareMathOperator{\Var}{Var}

\newtheoremstyle{plain}{}{}{\itshape}{}{\bfseries}{.}{0.5em}{}
\newtheoremstyle{remark}{}{}{}{}{\itshape}{.}{0.5em}{}

\theoremstyle{plain}
\newtheorem{theorem}{Theorem}
\newtheorem{proposition}{Proposition}
\newtheorem{corollary}{Corollary}
\newtheorem{lemma}{Lemma}

\theoremstyle{remark}
\newtheorem*{remark}{Remark}

\title{Uniform-in-time concentration in two-layer neural networks via transportation inequalities}

\author[1]{Arnaud Guillin\thanks{\href{mailto:arnaud.guillin@uca.fr}{\texttt{arnaud.guillin@uca.fr}}}}
\author[1]{Boris Nectoux\thanks{\href{mailto:boris.nectoux@uca.fr}{\texttt{boris.nectoux@uca.fr}}}}
\author[1]{Paul Stos\thanks{\href{mailto:paul.stos@uca.fr}{\texttt{paul.stos@uca.fr}}}}
\affil[1]{\small\textsc{Université Clermont-Auvergne}, CNRS UMR 6620, \textit{LMBP}, Clermont-Ferrand, France}

\date{}

\begin{document}

\maketitle

\begin{abstract}
  We quantify, \emph{uniformly over time} and with high probability, the discrepancy between the predictions of a two-layer neural network trained by stochastic gradient descent (SGD) and their mean-field limit, for quadratic loss and ridge regularization. As a key ingredient, we establish $T_p$ transportation inequalities ($p \in \{1,2\}$) for the law of the SGD parameters, with explicit constants independent of the iteration index. We then prove uniform-in-time concentration of the empirical parameter measure around its mean-field limit in the Wasserstein distance $\W_1$, and we translate these bounds into prediction-error estimates against a fixed test function $\varphi$. We also derive analogous concentration bounds in the sliced-Wasserstein distance $\SW_1$, leading to dimension-free rates. 
\end{abstract}

\section{Introduction}

Neural networks have become a standard tool for a variety of high-dimensional prediction problems~\cite{DeepLearning,LCBH,Schmidhuber}, and their empirical success has fueled substantial interest in understanding the theoretical mechanisms underpinning their behavior~\cite{Suh-Cheng,SSBD,LTFP}. 

In this paper we focus on the classical framework of supervised learning for a regression task.
We observe a stream of independent data points $(x_k, y_k) \in \R^d \times \R$ drawn from a common (unknown) distribution $\pi$. The goal is to learn a predictor $\hat{y}$ of the label $y$ that generalizes to previously unseen features $x$. In a one-hidden-layer neural network of \emph{width} $N$, this predictor is parametrized by $\theta = (\theta^1, \dots, \theta^N) \in \mathcal{E} = (\R^D)^N$ and takes the form
\[
  \hat{y}_\theta(x) = \frac1N \sum_{i=1}^N \sigma_*(\theta^i, x),
\]
where $\sigma_* : \R^D \times \R^d \to \R$ is an \emph{activation} function. Typically, $\theta^i = (a_i, b_i, w_i) \in \R^{d+2}$ and $\sigma_*(\theta^i, x) = a_i \sigma(w_i \cdot x + b_i)$ for some $\sigma : \R \to \R$ (e.g., sigmoid, tanh, ReLU). The factor $1/N$ is the standard \emph{mean-field} normalization of the network output~\cite{Descours}; it is made for convenience and can be absorbed into the activation by simple redefinition.

In principle, one would like to choose $\theta$ so as to minimize the \emph{population} risk
\[
  \mathcal R(\theta) = \E_\pi\!\left[\ell\big(y, \hat y_\theta(x)\big)\right],
\]
where $\ell: \R \times \R \to \R$ is a \emph{loss} function. For regression, we use the square loss $\ell(y, \hat{y}) = \frac12(y - \hat{y})^2$ out of simplicity, though other losses could be considered. 

In practice, since we do not have direct access to the data distribution $\pi$, a standard approach consists in minimizing an \emph{empirical} proxy $\widehat{\mathcal{R}}$ built from the observed samples via an optimization algorithm such as stochastic gradient descent (SGD) or its variants~\cite{Bottou,BCN}.

In the present setting, parameters are initialized i.i.d.\ from $\mu_0 \in \mathcal{P}(\R^D)$ and updated according to \emph{online} SGD with fixed learning rate $\alpha > 0$: for every $i \in \{1, \dots, N\}$,
\begin{equation}\label{eq:SGD}
  \theta_{k+1}^i = 
  (1-\gamma\lambda)\theta_k^i 
  + \gamma\big(y_{k+1} - \hat{y}_{\theta_k}(x_{k+1})\big) 
  \nabla_{\theta^i}\sigma_*(\theta_k^i,x_{k+1}),
\end{equation}
where $\gamma = \alpha / N$. The term $-\gamma\lambda$ corresponds to a \emph{ridge} regularization of the empirical risk
\[
  \widehat{\mathcal{R}}_\lambda(\theta, x, y)
  = \frac{1}{2} \big(y - \hat{y}_\theta(x)\big)^2
  + \frac{\lambda}{2}\|\theta\|_2^2
\]
which will be useful for our analysis below.

In the \emph{online} setting, each data point is used only once and never revisited. This idealization is standard in the learning literature and is relevant whenever data are abundant or streaming, so that repeated reuse of a fixed finite dataset is not the leading phenomenon~\cite{Bottou,Saad-Solla}. 

Providing quantitative guarantees for the training dynamics of neural networks remains challenging, even in the basic two-layer architecture trained by stochastic gradient descent. This difficulty stems not only from high dimensionality, but also from the \emph{feedback} structure of learning: at each iteration, each parameter is updated through a gradient that depends on the current network output, while this output aggregates the contributions of all $N$ neurons and is evaluated on random data samples. As a result, SGD produces a strongly coupled stochastic evolution in parameter space $\mathcal{E}$. 

For wide two-layer networks, a particularly fruitful approach has been to model training as a large interacting particle system and to track it via the \emph{empirical distribution} of the neurons' parameters~\cite{RVE}, or its rescaled-in-time version~\cite{Descours,DGMN}, defined for all $t \ge 0$ by
\[
  \mu_t^N = \frac1N \sum_{i=1}^N \delta_{\theta^i_{\lfloor Nt \rfloor}} \in \mathcal{P}(\R^D).
\]
As the number of neurons $N \to \infty$, this empirical process converges in probability (in an appropriate topology) to a deterministic evolution $(\bar\mu_t)_{t \ge 0}$ characterized as the solution to a nonlinear measure-valued equation~\cite{MMN,SSa,DGMN} (see~\eqref{eq:MF}). This asymptotic description yields a more tractable macroscopic model for the learning dynamics, and has proved useful for studying fluctuations~\cite{SSb,DGMN} and for analyzing stability and long-time behavior~\cite{MMN,Chizat-Bach}.

Existing quantitative finite-width bounds are typically derived on a \emph{fixed} horizon $[0, T]$, with constants deteriorating (often exponentially) as $T$ grows; see, e.g.,~\cite[Theorem~3]{MMN}. This is unsatisfactory in practice, since learning algorithms are routinely run for many iterations. 

In this work, we address this gap by quantifying \emph{uniformly over} $t \ge 0$ the deviation of the empirical parameter measure $\mu_t^N$ from its mean-field limit $\bar\mu_t$, under an explicit \emph{contractivity} condition induced by ridge regularization. More precisely, we study three complementary levels of discrepancies:
\begin{itemize}
  \item $\Delta_t^{N,\varphi} = \big|\langle \varphi,\mu_t^N \rangle - \langle \varphi,\bar\mu_t \rangle\big|$ for a \emph{fixed} $1$-Lipschitz test function $\varphi : \R^D \to \R$. In particular, assuming Lipschitz activation (assumption~\ref{assump:A1} below), for any fixed input $x \in \R^d$, taking $\varphi = \sigma_*(\cdot, x)$ translates any bound on $\Delta_t^{N,\varphi}$ into a bound on the network output.

  \item $\Delta_t^N = \W_1(\mu_t^N, \bar\mu_t)$, where $\W_1$ is the classical $1$--\emph{Wasserstein} distance on $\mathcal{P}_1(\R^D)$ (the set of Borel probability measures on $\R^D$ with finite first moment). Recall the Kantorovich--Rubinstein duality formula for $\W_1$ (see, e.g.,~\cite[Remark 6.5]{Old-New}):
  \begin{equation}\label{eq:Kantorovich-Rubinstein}
    \W_1(\mu,\nu) = \sup_{\|\psi\|_{\Lip} \le 1}\big|\langle \psi,\mu\rangle-\langle \psi,\nu\rangle\big|.
  \end{equation}
  Hence $\Delta_t^N$ simultaneously controls $\Delta_t^{N,\varphi}$ for \emph{all} $1$-Lipschitz $\varphi$. 

  \item $\Delta_t^{N,\s} = \SW_1(\mu_t^N, \bar\mu_t)$, where $\SW_1$ is the $1$--\emph{sliced-Wasserstein} distance, defined by
  \[
    \SW_1(\mu, \nu)
    = \int_{\S^{D-1}} \W_1\big({P_u}_\#\mu, {P_u}_\#\nu\big)\,\zeta(du),
    \qquad P_u(w) = u \cdot w,
  \]
  with $\zeta$ the uniform probability measure on the unit sphere $\S^{D-1}$. The sliced analog $\SW_1$ is widely used in high dimension and is computationally tractable~\cite{BRPP,Peyre-Cuturi,NDJBS} as it reduces comparisons to one-dimensional projections, while still probing a rich family of observables (all linear projections). This makes $\Delta_t^{N,\s}$ a useful compromise between the very targeted discrepancy $\Delta_t^{N,\varphi}$ and the fully global metric $\Delta_t^N$.
\end{itemize}
When the particular choice is immaterial, we write 
\[
  \Delta_t^{N,\circ} \in \{\Delta_t^{N,\varphi}, \Delta_t^N, \Delta_t^{N,\s}\}.
\]
We refer to~\cite{Santambrogio,Old-New,Peyre-Cuturi} for background on Wasserstein distances and optimal transport (see also~\cite{OTML} for a more machine-learning-oriented introduction), and to~\cite{BRPP,NDCKSS,NDJBS} for sliced-Wasserstein and related sliced divergences.

Our analysis hinges on two steps. 
First, under regularity conditions on the initialization law $\mu_0$ and the data distribution $\pi$, we establish a $T_p$ \emph{transportation} inequality~\cite{Gozlan-Leonard} ($p \in \{1,2\}$) for the law $\nu_k = \L(\theta_k)$ of the parameter vector at iteration $k$, with explicit constants independent of $k$. This is the content of~\cref{prop:Tp}. For both $p=1$ and $p=2$, this implies sub-Gaussian tails~\eqref{eq:tailp} for centered Lipschitz observables of the parameters (hence of $\mu_t^N$), with constants that do not deteriorate with time. Instantiated with suitable Lipschitz functionals of $\mu_t^N$, this first step yields high-probability deviation bounds around the mean for our three discrepancies of interest (see \cref{cor:GC}). 
We then bound the remaining bias term $\E[\Delta_t^{N,\circ}]$ uniformly over $t \ge 0$ by comparing the interacting SGD particles to an i.i.d.\ mean-field system through a synchronous coupling. 
This yields a decomposition into a \emph{dynamic} error along trajectories, handled via a standard \emph{propagation-of-chaos} argument~\cite{Sznitmann} (\cref{prop:PoC,cor:PoC}), and a \emph{static} i.i.d.\ sampling error for empirical measures, controlled by classical empirical rates in Wasserstein distance~\cite{Fournier-Guillin} (\cref{prop:FG,prop:FG2}). 
Combining the two steps, we obtain a \emph{nonasymptotic}, \emph{uniform-in-time} concentration bound around the mean-field limit. This constitutes our main result, presented as~\cref{thm:concentration}. These results are established under standard boundedness and Lipschitz assumptions on the activation $\sigma_*$. 
The final section explains how to relax boundedness by allowing at most linear growth in the parameter variable on the support of the data, hence covering more general ReLU-type activations, via localization.

\paragraph{Related works.}
Mean-field approximations for the online training of wide two-layer neural networks were introduced in~\cite{MMN}, where the limiting dynamics is derived together with finite-width bounds on a fixed time horizon $[0, T]$. A trajectorial law of large numbers and central limit theorem for the empirical parameter measure on any fixed time interval were proved in~\cite{SSa,SSb}, and later extended at the level of the full trajectory (under mini-batching and noisy regimes) in~\cite{DGMN}. For complementary analyses of the mean-field dynamics, including stability and long-time behavior under various structural assumptions, we refer to~\cite{RVE,Chizat-Bach,MMN}. We also mention the quantitative propagation-of-chaos analysis for (continuous-time) SGD dynamics developed in~\cite{DBDFS}.

Beyond fixed-horizon bounds, obtaining \emph{time-uniform} control typically requires some form of \emph{dissipativity} (e.g., contractivity in a suitable metric) to prevent error accumulation; see, e.g.,~\cite{DEGZ,LLF,Delarue-Tse} in the setting of weakly interacting diffusions. Uniform-in-time exponential bounds for Wasserstein fluctuations have been established for related particle systems (including heterogeneous mean-field interactions) in~\cite{Bayraktar-Wu}, extending transport-based concentration estimates for empirical measures~\cite{BGV}; see also the general mean-field particle-model framework of~\cite{DMR}.

In a different direction, several works extend mean-field ideas to deep architectures and other infinite-width regimes. Rigorous mean-field formalisms for multilayer networks are developed in~\cite{SSdeep,Nguyen,Nguyen-Pham,FLYZ}. Mean-field viewpoints have also been proposed for transformer architectures, where token dynamics can be interpreted as interacting particle systems~\cite{Rigollet,CLPR}.

From a more geometric perspective, the classical i.i.d.\ Wasserstein sampling-error bounds of Fournier and Guillin~\cite{Fournier-Guillin} have recently been sharpened in~\cite{DFM}, where one obtains similar optimal rates with the ambient dimension replaced by the covering dimension of the support. This suggests a potential direction for future work in our setting: if the mean-field limit concentrates on a low-dimensional subset of parameter space, then identifying an effective support and estimating (or upper-bounding) its covering dimension could yield strictly improved rates.

\bigskip

During the preparation of this manuscript, we became aware of~\cite{Vuckovic}, which establishes quantitative propagation-of-chaos bounds for general McKean-type nonlinear Markov chains and discusses long-time behavior under additional structural assumptions. While the settings and techniques differ, both works support the viewpoint that uniform-in-time propagation of chaos is fundamentally tied to an underlying contractive structure.

\bigskip

The rest of this paper is organized as follows. Our standing assumptions and main results are stated in~\cref{sec:main-results}.~\cref{sec:proofs} contains the proofs. Finally,~\cref{sec:extension} presents an extension of the previous results to linearly growing activations.

\paragraph{Notation.}
All random variables are defined on a common probability space $(\Omega,\F,\P)$. For a random variable $Z$, $\L(Z)$ denotes the law of $Z$. We use $|\cdot|$ both for the absolute value on $\R$ and the Euclidean norm on $\R^d$ or $\R^D$. The scalar product on $\R^D$ is denoted by $v \cdot w$. We write $|\cdot|_\infty$ for the sup norm on $\mathcal{Z} = \R^d \times \R$, and $\|\cdot\|_p$ ($p \in \{1,2\}$) for the usual $\ell^p$ product norm on $\mathcal{E} = (\R^D)^N$:
\[
  \|\theta\|_p = \Big(\sum_{i=1}^N |\theta^i|^p\Big)^{1/p},\qquad
  \theta = (\theta^1, \dots, \theta^N) \in \mathcal{E}.
\]
Given a metric space $(\X,\d)$ and a function $f : \X \to \R$, we define (whenever finite)
\[
  \|f\|_\infty = \sup_{x \in \X} |f(x)|,\qquad
  \|f\|_{\Lip} = \sup_{x \ne y} \frac{|f(x) - f(y)|}{\d(x,y)}.
\] 
For measurable $\varphi : \R^D \to \R$ and $\mu \in \mathcal{P}(\R^D)$, we use the bracket notation
\[
  \langle \varphi,\mu \rangle = \int_{\R^D} \varphi\,d\mu.
\]
In particular, for any input $x \in \R^d$ and $t \ge 0$, we have
\[
  \hat{y}_{\theta_{\lfloor Nt \rfloor}}(x) = \langle \sigma_*(\cdot, x),\mu_t^N \rangle.
\]
Throughout this work, $C > 0$ denotes a generic constant (independent of $N$, $t$ and $p$) whose value may change from line to line.

\section{Main results}\label{sec:main-results}

\subsection{Assumptions} 

We define the following set of assumptions.

\begin{enumerate}[label={(A\arabic*)}, font=\bfseries]
  \item\label{assump:A1}
  The activation function $\sigma_* \in \mathcal{C}^1(\R^D \times \R^d)$; moreover, $\sigma_*$ is uniformly bounded with bounded mixed gradient, namely
  \[
    \|\sigma_*\|_\infty \le B,\qquad
    \|\nabla_{(\theta^i, x)}\,\sigma_*\|_\infty \le M.
  \]

  \item\label{assump:A2} Labels are almost surely bounded: $|y| \le A$ $\pi$-a.s.

  \item\label{assump:A3}
  For each $\theta^i \in \R^D$, the map $x \in \R^d \mapsto \nabla_{\theta^i}\sigma_*(\theta^i,x)$ is $L_x$-Lipschitz uniformly in $\theta^i$; \newline
  for each $x \in \R^d$, the map $\theta^i \in \R^D \mapsto \nabla_{\theta^i}\sigma_*(\theta^i,x)$ is $L_\theta$-Lipschitz uniformly in $x$.
\end{enumerate}

Assumptions~\ref{assump:A1}--\ref{assump:A3} are standard in the learning literature: for instance, \cite{MMN} assumes bounded activations and labels together with bounded-Lipschitz drift terms, and uses a sub-Gaussian control on $\nabla_{\theta^i}\sigma_*(\theta^i, \cdot)$ to handle high-dimensional inputs. Our conditions are slightly more direct (uniform bounds and Lipschitz constants), which makes the stability constants explicit and keeps the subsequent concentration arguments transparent. The main limitation is that \ref{assump:A1} excludes non-smooth or unbounded activations such as ReLU. This condition will be relaxed in the final~\cref{sec:extension}.

\bigskip

Under these assumptions, we collect the following constants, used repeatedly in what follows.
\begin{equation}\label{eq:constants}
  \begin{aligned}
    &K = \alpha\big((A+B)L_x + M^2 + M\big),
    \qquad
    \lambda_\star = (A+B)L_\theta + M^2,
    \qquad
    L_N = |1 - \gamma\lambda| + \gamma\lambda_\star,\\[0.5em]
    &\qquad\qquad\qquad C_\star = 8M^4 + 4(\lambda + (A+B)L_\theta)^2,
    \qquad
    N_\star = \bigg\lceil \frac{4\alpha C_\star}{\lambda - \lambda_\star} \bigg\rceil.
  \end{aligned}
\end{equation}

\subsection{Transportation inequalities for the SGD dynamics}
\label{sec:Tp}

Transportation inequalities relate relative entropy to optimal transport and offer a direct route to proving Gaussian-type concentration for Lipschitz observables. Here, they form the first ingredient of our analysis: we show that the law $\nu_k = \L(\theta_k)$ of the parameter vector at iteration $k$ satisfies a $T_p$ transportation inequality (for $p \in \{1,2\}$) with constants independent of $k$ (\cref{prop:Tp}), which in turn yields uniform sub-Gaussian concentration along the SGD Markov chain (\cref{cor:GC}). We briefly recall below the definition of $T_p$ and the few properties we will use; for general background about this type of inequalities and their concentration of measure consequences, we refer for instance to~\cite{Ledoux,MCTCFI,Gozlan-Leonard} and the references therein.

Given a Polish metric space $(\X,\d)$ and $p \ge 1$, we say that a probability measure $\mu \in \mathcal{P}(\X)$ satisfies an $L^p$--\emph{transportation} inequality on $(\X,\d)$ with constant $C > 0$, and we write $\mu \in T_p(C)$, if for all $\nu \in \mathcal{P}(\X)$,
\begin{equation}\label{eq:Tp}
  \W_p(\nu, \mu) \le \sqrt{2C\,H(\nu\,|\,\mu)},
\end{equation}
where $\W_p$ denotes the usual $p$-\emph{Wasserstein} distance on $\mathcal{P}_p(\X)$ (the set of Borel probability measures with finite $p$-th moment), and
\[
  H(\nu\,|\,\mu) = 
  \left\{
    \begin{array}{cl}
      \displaystyle\int_{\X} \log\frac{d\nu}{d\mu}\,d\nu & \text{if } \nu \ll \mu\\
      +\infty & \text{otherwise}
    \end{array}
  \right.
\]
is the \emph{relative entropy} of $\nu$ with respect to $\mu$.

Among all values of $p \ge 1$, the cases $p = 1$ and $p = 2$ have attracted the most attention. For $T_1$, Bobkov--G\"otze's theorem~\cite{Bobkov-Gotze} gives an equivalence between~\eqref{eq:Tp} and Gaussian concentration: $\mu\in T_1(C)$ on $(\X,\d)$ if and only if, for every Lipschitz $f : (\X,\d) \to \R$, $f$ is $\mu$--integrable and
\begin{equation}\label{eq:GC}
  \forall \xi \in \R,\quad
  \E_\mu \big[e^{\xi(f - \E_\mu[f])}\big] \le 
  \exp\!\left(\frac{C\xi^2}{2} \|f\|^2_{\Lip}\right).
\end{equation}
In particular, if $\mu$ is compactly supported with, say $\supp\mu \subset B(x_0, R)$ for some $x_0 \in \X$ and $R > 0$, then~\eqref{eq:GC} immediately follows from Hoeffding's lemma, implying $\mu \in T_1(R^2/4)$. More generally, $T_1$ is also equivalent to an exponential integrability condition on $\d(x,x_0)^2$, see~\cite[Theorem~2.3]{DGW}. By Chernoff's method,~\eqref{eq:GC} in turn yields the sub-Gaussian tail bound
\begin{equation}\label{eq:tail}
  \forall r > 0,\quad
  \P\big(|f(X) - \E[f(X)]|\ge r\big) 
  \le 2\exp\!\bigg(-\frac{r^2}{2C\|f\|_{\Lip}^2}\bigg),\qquad 
  X \sim \mu.
\end{equation}
Note that $T_2(C)$ is stronger than $T_1(C)$ (since $\W_1 \le \W_2$), and therefore also implies Gaussian concentration for Lipschitz functionals. In contrast with $T_1$, compact support alone does not guarantee $T_2$ in general; proving $T_2$ typically requires additional structure, such as a logarithmic Sobolev inequality~\cite{Otto-Villani}.

We will repeatedly use the stability of $T_p$ under Lipschitz pushforward~\cite[Lemma~2.1]{DGW}: if $\mu \in T_p(C)$ on $(\X,\d)$ and $\Psi : (\X,\d) \to (\Y,\d')$ is $L$-Lipschitz, then $\Psi_\#\mu \in T_p(L^2 C)$ on $(\Y,\d')$.

We will also rely on the following i.i.d.\ tensorization identities. For $n \ge 1$, equip $\X^n$ with the product metrics
\[
  \d_1(x,y) = \sum_{i=1}^n \d(x_i,y_i),\qquad
  \d_2(x,y) = \Big(\sum_{i=1}^n \d(x_i,y_i)^2\Big)^{1/2}.
\]
If $\mu \in T_1(C)$ on $(\X,\d)$, then $\mu^{\otimes n} \in T_1(nC)$ on $(\X^n,\d_1)$ and
if $\mu \in T_2(C)$ on $(\X,\d)$, then $\mu^{\otimes n} \in T_2(C)$ on $(\X^n,\d_2)$; see~\cite[Proposition~2.3]{MCTCFI} and the remark after~\cite[Proposition~1.3]{MCTCFI}.

\bigskip

All of our statements below hold for either choice of $p \in \{1,2\}$. To propagate a $T_p$ inequality along the SGD Markov chain, it is natural to assume a $T_p$ inequality at initialization and for the data distribution:

\begin{enumerate}[label=(A\arabic*$_p$), font=\bfseries, ref=(A\arabic*$_p$)]
  \setcounter{enumi}{3} 
  \item\label{assump:A4p} The initialization law $\mu_0$ satisfies $\mu_0 \in T_p\big(C_0^{(p)}\big)$ on $(\R^D, |\cdot|)$, and the data distribution $\pi$ satisfies $\pi \in T_p\big(C_\pi^{(p)}\big)$ on $(\R^d \times \R, |\cdot|_\infty)$.
\end{enumerate}
Under assumption~\ref{assump:A4p} and the one-step contractivity condition $L_N<1$ (see~\eqref{eq:constants}), we can propagate $T_p$ inequalities uniformly over the SGD iterates, in the spirit of~\cite[Theorem~2.5]{DGW} and~\cite[Theorem~1.2]{Blower-Bolley}.

\begin{proposition}\label{prop:Tp}
  Fix $p\in\{1,2\}$. Assume \ref{assump:A1}--\ref{assump:A3}, \ref{assump:A4p}, and $L_N < 1$.
  Then for all $k\in\N$, $\nu_k \in T_p\big(C_N^{(p)}\big)$ on $(\mathcal{E},\|\cdot\|_p)$, with the explicit constants
  \[
    C_N^{(1)} = NC_0^{(1)} + \dfrac{K^2}{(1 - L_N^2)} C_\pi^{(1)},\qquad
    C_N^{(2)} = \frac{C_0^{(2)} \vee K^2\,C_\pi^{(2)} N^{-1}}{(1 - L_N)^2}.
  \]
\end{proposition} 

\begin{remark}
  Under our scaling $\gamma = \alpha / N$, one typically has $\gamma \lambda_\star < 1$ for $N$ large enough, in which case the condition $L_N < 1$ amounts to the admissible interval $\lambda_\star < \lambda < 2\gamma^{-1} - \lambda_\star$ for the regularization coefficient. The lower bound $\lambda_\star < \lambda$ ensures the regularization dominates the interaction term, making the SGD update \emph{contractive} in the parameter variable (see~\cref{lem:Lipschitz}~\ref{Lipschitz:ii} below). The upper bound $\lambda < 2\gamma^{-1} - \lambda_\star$ is a discretization artifact of~\eqref{eq:SGD}, and is essentially nonbinding for wide networks.
\end{remark}

As a direct consequence of~\cref{prop:Tp} (see indeed~\eqref{eq:tail}), for any Lipschitz $f : (\mathcal{E},\|\cdot\|_p) \to \R$ and $p \in \{1,2\}$, it holds for every $k \in \N$,
\begin{equation}\label{eq:tailp}
  \forall r > 0,\quad 
  \P\big(|f(\theta_k)-\E[f(\theta_k)]|\ge r\big)
  \le 2\exp\!\bigg(-\frac{r^2}{2C_N^{(p)}\|f\|_{\Lip,p}^2}\bigg)
\end{equation}
where
\[
  \|f\|_{\Lip,p} = \sup_{\theta \ne \theta'} \frac{|f(\theta)-f(\theta')|}{\|\theta-\theta'\|_p}.
\]
For $t \ge 0$ and $\theta = (\theta^1, \dots, \theta^N) \in \mathcal{E}$, let 
\[
  \mu_\theta = \frac1N \sum_{i=1}^N \delta_{\theta^i} \in \mathcal{P}(\R^D),
\]
and define for any $1$-Lipschitz $\varphi: \R^D \to \R$,
\[
  f_t^\varphi(\theta) = \big|\langle \varphi, \mu_\theta \rangle - \langle \varphi, \bar\mu_t \rangle\big|,\qquad
  f_t(\theta) = \W_1(\mu_\theta, \bar\mu_t),\qquad
  f_t^\s(\theta) = \SW_1(\mu_\theta, \bar\mu_t).
\]
By construction, evaluating along the SGD iterates at iteration $k = \lfloor Nt \rfloor$, we thus recover 
\begin{equation}\label{eq:Lipschitz-observables}
  f_t^\varphi(\theta_{\lfloor Nt\rfloor}) = \Delta_t^{N,\varphi},\qquad
  f_t(\theta_{\lfloor Nt\rfloor}) = \Delta_t^N,\qquad
  f_t^\s(\theta_{\lfloor Nt\rfloor}) = \Delta_t^{N,\s}.
\end{equation}
Using a canonical coupling (pairing the $i$-th atoms), for any $\theta, \theta' \in \mathcal{E}$,
\[
  \W_1(\mu_\theta,\mu_{\theta'}) \le \frac1N \sum_{i=1}^N |\theta^i - \theta'^i| = \frac1N \|\theta - \theta'\|_1 \le \frac{1}{\sqrt{N}} \|\theta - \theta'\|_2.
\]
By Kantorovich--Rubinstein duality~\eqref{eq:Kantorovich-Rubinstein}, any $1$-Lipschitz $\varphi$ satisfies $\big|\langle \varphi,\mu\rangle-\langle \varphi,\nu\rangle\big|\le \W_1(\mu,\nu)$,
and since each projection $P_u$ is $1$-Lipschitz, we also have $\SW_1(\mu,\nu) \le \W_1(\mu,\nu)$. Hence by triangular inequality, $f_t^\varphi$, $f_t$, and $f_t^\s$ are all $N^{-1/p}$--Lipschitz on $(\mathcal{E}, \|\cdot\|_p)$, for $p \in \{1,2\}$.

Although the constant $C_N^{(1)}$ in~\cref{prop:Tp} may scale like $N$, notice that this does not affect~\eqref{eq:tailp} since for both $p=1$ and $p=2$, we have $C_N^{(p)}N^{-2/p} = O(N^{-1})$. Hence, combining~\eqref{eq:tailp} with~\eqref{eq:Lipschitz-observables} yields the following uniform deviation inequality.

\begin{corollary}\label{cor:GC}
  Fix $p\in\{1,2\}$ and assume~\ref{assump:A1}--\ref{assump:A3}, \ref{assump:A4p} and $L_N < 1$. Then there exists $C > 0$ (independent of $N$ and $t$) such that for any $t \ge 0$ and $\delta \in (0, 1)$, with probability $1 - \delta$, 
  \[
    \big|\Delta_t^{N, \circ} - \E[\Delta_t^{N, \circ}]\big| \le C \sqrt{\frac{\log(2/\delta)}{N}}.
  \]
\end{corollary}

\subsection{Uniform bias decay and concentration around mean-field}

\cref{cor:GC} above provides a concentration inequality for $\Delta_t^N$ (and similarly for $\Delta_t^{N,\varphi}$ and $\Delta_t^{N,\s}$) around its mean. It thus remains to control this bias uniformly in time in order to recover a concentration inequality for $\mu_t^N$ around $\bar\mu_t$. 

Assumption~\ref{assump:A4p} introduced above is only needed to derive~\cref{prop:Tp} and~\cref{cor:GC}. For the bias estimates below, it is enough to assume a finite $(2+\varepsilon)$-moment at initialization:
\begin{enumerate}[label={(A\arabic*)}, font=\bfseries, ref=(A\arabic*)]
  \setcounter{enumi}{4}
  \item\label{assump:A5} There exists $q > 2$ such that $\E\big[|\theta^1_0|^q\big] < \infty$.
\end{enumerate}
In particular, for either choice of $p \in \{1,2\}$, assumption~\ref{assump:A4p} implies~\ref{assump:A5} (since both $T_1$ and $T_2$ imply exponential integrability and therefore moments of all orders).

\begin{proposition}\label{prop:bias}
  Assume~\ref{assump:A1}--\ref{assump:A3},~\ref{assump:A5} and $L_N < 1$, $N \ge N_\star$ (see~\eqref{eq:constants}). Then there exists $C > 0$ (independent of $N$ and $t$) such that
  \[
    \sup_{t \ge 0} \E[\Delta_t^{N, \circ}] \le C\,\kappa_N,
  \]
  where
  \[
    \kappa_N =
    \begin{cases}
      N^{-1/2} & \text{for} \ \Delta_t^{N,\varphi} \text{and} \ \Delta_t^{N,\s},\\
      N^{-1/(1 \vee D)} + N^{-1/2} & \text{for } \Delta_t^N.
    \end{cases}
  \]
\end{proposition} 

\begin{remark}
  The only dimension-dependent term in $\kappa_N$ appears for the Wasserstein discrepancy $\Delta_t^N = \W_1(\mu_t^N,\bar\mu_t)$, through the rate $N^{-1/(1\vee D)}$. This term originates from the optimal i.i.d.\ sampling error of empirical measures in $\W_1$ on $\R^D$~\cite[Theorem~1]{Fournier-Guillin}, and reflects the intrinsic complexity of optimal matching in high dimension, also known as the \emph{curse of dimensionality}. By contrast, both $\Delta_t^{N,\varphi}$ and $\Delta_t^{N,\s}$ enjoy dimension-free $N^{-1/2}$ rates, since they reduce to one-dimensional (or scalar) averages. Sharper dimension-specific rates and possible logarithmic corrections are available in low dimension (see the discussion after Theorem 1 in~\cite{Fournier-Guillin}), but the above bound is sufficient for our purposes.
\end{remark}

Finally, combining~\cref{cor:GC} with~\cref{prop:bias}, we obtain the following uniform-in-time concentration inequality between the empirical measure of SGD parameters and its mean-field limit.

\begin{theorem}\label{thm:concentration}
  Assume~\ref{assump:A1}--\ref{assump:A3}, \ref{assump:A4p} (for either $p=1$ or $p=2$), and $L_N < 1$, $N \ge N_\star$ (see~\eqref{eq:constants}). Then there exists $C > 0$ (independent of $N$ and $t$) such that for any $t \ge 0$ and $\delta \in (0, 1)$, with probability $1 - \delta$, 
  \[
    \Delta_t^{N, \circ} \le C\,\kappa_N + C\sqrt{\frac{\log(2/\delta)}{N}}.
  \]
  with $\kappa_N \to 0$ as $N \to \infty$ given in~\cref{prop:bias}.
\end{theorem}

\paragraph{Application to the network output.}
Fix $x \in \R^d$ and let $\varphi_x = \sigma_*(\cdot, x)$. Under~\ref{assump:A1}, $\varphi_x$ is $M$-Lipschitz uniformly in $x$. Applying~\cref{thm:concentration} to $\Delta_t^{N,\varphi}$ for $t \ge 0$, with $\varphi = \varphi_x / M$ thus yields that for any $\delta > 0$, with probability $1 - \delta$,
\[
  \big|\hat{y}_{\theta_{\lfloor Nt \rfloor}}(x) - \langle \sigma_*(\cdot, x), \bar\mu_t\rangle\big|
  \le C\sqrt{\frac{1 + \log(2/\delta)}{N}}.
\]

\section{Proofs}
\label{sec:proofs}

\subsection{SGD dynamics}

We start by rewriting the SGD recursion in a form that makes its Markov structure explicit, since Proposition~\ref{prop:Tp} is ultimately a statement about transport inequalities along the resulting Markov chain. The proof combines a Lipschitz control of the one-step update (\cref{lem:Lipschitz}) with stability results for $T_p$ inequalities under contracting Markov kernels. For $p = 1$, propagation of $T_1$ is obtained by an elementary induction based on the Laplace transform characterization~\eqref{eq:GC}. For $p=2$, we instead rely on a dependent tensorization result for contracting Markov chains~\cite[Theorem~1.2]{Blower-Bolley} applied at the level of the full trajectory, and then project onto the last coordinate to recover a $T_2$ inequality uniform in the iteration index.

For $\theta = (\theta^1, \dots, \theta^N) \in \mathcal{E} = (\R^D)^N$ and $z = (x,y) \in \mathcal{Z} = \R^d \times \R$, define the \emph{one-step drift}
\begin{equation}\label{eq:SGD-drift}
  \begin{aligned}
    F_\lambda^i(\theta,z)
    &= \big(y - \hat y_\theta(x)\big)\,\nabla_{\theta^i} \sigma_*(\theta^i,x) - \lambda\,\theta^i,\qquad 
    i \in \{1, \dots, N\},\\
    F_\lambda(\theta, z) 
    &= \big(F_\lambda^1(\theta, z), \dots, F_\lambda^N(\theta, z)\big)\in\mathcal E,
  \end{aligned}
\end{equation}
and the associated update map $\Phi(\theta, z) = \theta + \gamma F_\lambda(\theta, z)$. Then~\eqref{eq:SGD} is equivalently
\begin{equation}\label{eq:SGD-update}
  \theta_{k+1} = \theta_k + \gamma F_\lambda(\theta_k, z_{k+1}) = \Phi(\theta_k, z_{k+1})
\end{equation}
with $z_{k+1} = (x_{k+1}, y_{k+1}) \in \mathcal{Z}$. Since $(z_k)_{k \ge 1}$ are i.i.d., the parameter sequence $(\theta_k)_{k\in\N}$ thus forms a time-homogeneous Markov chain on $\mathcal E$ with transition kernel
\begin{equation}\label{eq:kernel}
  \forall \theta\in\mathcal{E},\quad
  P(\theta, d\theta') = \big[\Phi(\theta,\cdot)_\#\pi\big](d\theta').
\end{equation}

\begin{lemma}\label{lem:Lipschitz}
  Assume~\ref{assump:A1}--\ref{assump:A3}. Then, for any $\theta, \theta' \in \mathcal{E}$ and $z, z' \in \supp\pi$, 
  \begin{enumerate}[label={\normalfont{(\roman*)}}]
    \item\label{Lipschitz:i} $\|\Phi(\theta, z) - \Phi(\theta, z')\|_2 \le \dfrac{K}{\sqrt{N}}\,|z - z'|_\infty$,\ and \
    $\|\Phi(\theta, z) - \Phi(\theta, z')\|_1 \le K\,|z - z'|_\infty$,
    \item\label{Lipschitz:ii} $\|\Phi(\theta, z) - \Phi(\theta', z)\|_p \le L_N\,\|\theta - \theta'\|_p$ \ for any $p \in \{1, 2\}$,
  \end{enumerate}
  with constants $K$, $L_N$ defined in~\eqref{eq:constants}.
\end{lemma}

\begin{proof}
  \ref{Lipschitz:i} Let $\theta, \theta' \in \mathcal{E}$ and $z =(x, y),\ z' = (x', y') \in \supp\pi \subset \R^d \times [-A, A]$ (by~\ref{assump:A2}).\\ 
  For every $i \in \{1, \dots, N\}$, by triangular inequality,
  \begin{align*}
    |\Phi^i(\theta, z) - \Phi^i(\theta, z')| 
    &= \gamma\big|(y - \hat{y}_\theta(x)) \nabla_{\theta^i}\sigma_*(\theta^i, x) - (y' - \hat{y}_\theta(x')) \nabla_{\theta^i}\sigma_*(\theta^i, x')\\
    &\qquad\ \pm (y - \hat{y}_\theta(x)) \nabla_{\theta^i}\sigma_*(\theta^i, x')\big|\\
    &\le \gamma|y - y'||\nabla_{\theta^i}\sigma_*(\theta^i, x')| + \gamma|y - \hat{y}_\theta(x)||\nabla_{\theta^i}\sigma_*(\theta^i, x) - \nabla_{\theta^i}\sigma_*(\theta^i, x')|\\
    &\quad + \gamma|\hat{y}_\theta(x') - \hat{y}_\theta(x)||\nabla_{\theta^i}\sigma_*(\theta^i, x')|.
  \end{align*}
  By~\ref{assump:A1}, the map $x \in \R^d \mapsto \sigma_*(\theta^i, x)$ is $M$-Lipschitz (uniformly in $i$), hence $|\hat{y}_\theta(x') - \hat{y}_\theta(x)| \le M |x - x'|$. Using also~\ref{assump:A3}, we get
  \[
    |\Phi^i(\theta, z) - \Phi^i(\theta, z')| \le \gamma M\,|y - y'| + \gamma \big((A+B)L_x + M^2\big)|x - x'|. 
  \]
  It follows that 
  \begin{align*}
    \|\Phi(\theta, z) - \Phi(\theta, z')\|_2 
    \le \gamma \sqrt{N} \big(M\,|y - y'| + \big((A+B)L_x + M^2\big)|x - x'|\big)
    \le \frac{K}{\sqrt{N}}\,|z - z'|_\infty,
  \end{align*}
  and consequently
  \[
    \|\Phi(\theta, z) - \Phi(\theta, z')\|_1 \le \sqrt{N} \|\Phi(\theta, z) - \Phi(\theta, z')\|_2 \le K\,|z - z'|_\infty.
  \]

  \bigskip

  \noindent\ref{Lipschitz:ii} Again, by triangular inequality, for every $i \in \{1, \dots, N\}$,
  \begin{align*}
    |\Phi^i(\theta, z) - \Phi^i(\theta', z)|
    &= \big|(1 - \gamma\lambda) (\theta^i - \theta'^i) + \gamma (y - \hat{y}_\theta(x)) \nabla_{\theta^i}\sigma_*(\theta^i, x)\\
    &\quad  -  \gamma \big(y - \hat{y}_{\theta'}(x)\big) \nabla_{\theta^i}\sigma_*(\theta'^i, x) \pm \gamma (y - \hat{y}_\theta(x)) \nabla_{\theta^i}\sigma_*(\theta'^i, x)\big|\\
    &\le |1 - \gamma\lambda||\theta^i - \theta'^i| + \gamma|y - \hat{y}_\theta(x)||\nabla_{\theta^i}\sigma_*(\theta^i, x) - \nabla_{\theta^i}\sigma_*(\theta'^i, x) |\\
    &\quad + \gamma\big|\hat{y}_{\theta'}(x) - \hat{y}_\theta(x)\big||\nabla_{\theta^i}\sigma_*(\theta'^i, x)|.
  \end{align*}
  By~\ref{assump:A1}, the map $\theta^i \in \R^D \mapsto \sigma_*(\theta^i, x)$ is $M$-Lipschitz (uniformly in $x \in \R^D$), hence
  \[
    |\hat{y}_{\theta'}(x) - \hat{y}_\theta(x)| \le \frac1N \sum_{i=1}^N |\sigma_*(\theta'^i, x) - \sigma_*(\theta^i, x)| \le \frac{M}{N^{1/p}} \|\theta - \theta'\|_p.
  \]
  Using also~\ref{assump:A2} and~\ref{assump:A3}, we get
  \[
    |\Phi^i(\theta, z) - \Phi^i(\theta', z)|
    \le \big(|1 - \gamma\lambda| + \gamma (A+B)L_\theta \big)\,|\theta^i - \theta'^i| + \gamma \frac{M^2}{N^{1/p}}\,\|\theta - \theta'\|_p.
  \]
  It follows by Minkowski inequality that $\|\Phi(\theta, \cdot) - \Phi(\theta', \cdot)\|_p \le L_N\,\|\theta - \theta'\|_p$ for all $p \in \{1,2\}$.
\end{proof}

\medskip

With~\cref{lem:Lipschitz} in hand, we are now in position to pass from one-step stability to uniform-in-time concentration along the SGD iterates.

\begin{proof}[Proof of~\cref{prop:Tp}]
  \noindent\emph{Case $p = 1$.} The proof goes by induction on $k \in \N$. Since $\mu_0 \in T_1\big(C_0^{(1)}\big)$, by tensorization we have $\nu_0 = \mu_0^{\otimes N} \in T_1\big(NC_0^{(1)}\big)$ on $(\mathcal{E}, \|\cdot\|_1)$. Now, assume that $\nu_k \in T_1(c_k)$ for some $c_k > 0$ and $k \in \N$. Let us show that $\nu_{k+1} \in T_1(c_{k+1})$ with
  \begin{equation}\label{eq:ck}
    c_{k+1} = L_N^2 c_k + K^2 C_\pi^{(1)},
  \end{equation}
  where $K$ and $L_N$ are the Lipschitz constants of~\cref{lem:Lipschitz}. Here, we use the equivalence between $T_1$ and Gaussian concentration~\eqref{eq:GC}. Let $f : (\mathcal{E}, \|\cdot\|_1) \to \R$ be $1$-Lipschitz. Take $Z \sim \pi$, and let for $\theta \in \mathcal{E}$, $g(\theta) = \E[f(\Phi(\theta, Z))]$. By \cref{lem:Lipschitz}, it follows that $g$ is $L_N$-Lipschitz, and for each fixed $\theta \in \mathcal{E}$, the map $z \mapsto f(\Phi(\theta, z))$ is $K$-Lipschitz. Since $\pi \in T_1\big(C_\pi^{(1)}\big)$, the concentration bound~\eqref{eq:GC} gives, for all $\theta \in \mathcal{E}$,
  \[
    \forall \xi \in \R, \quad 
    \E\big[e^{\xi(f(\Phi(\theta, Z)) - g(\theta))}\big]
    \le \exp\!\bigg(\frac{C_\pi^{(1)}\xi^2}{2} K^2\bigg).
  \]
  Now, let $\theta = \theta_k \sim \nu_k$ be random and independent of $Z$, so that $\Phi(\theta_k, Z)$ is equal in distribution to $\theta_{k+1}$ (see~\eqref{eq:SGD-update}). Conditioning on $\theta_k$, we get for any $\xi \in \R$,
  \begin{align*}
    \E\big[e^{\xi(f(\theta_{k+1}) - \E[f(\theta_{k+1})])}\big]
    &= \E\Big[
      e^{\xi(g(\theta_k) - \E[g(\theta_k)])}\,
      \E\big[
        e^{\xi(f(\theta_{k+1}) - g(\theta_k))}\,\big|\,\theta_k
      \big]
    \Big]\\
    &\le \exp\!\bigg(\frac{C_\pi^{(1)}\xi^2}{2} K^2\bigg)\,
      \E\big[e^{\xi(g(\theta_k) - \E[g(\theta_k)])}\big].
  \end{align*}
  Since $\nu_k \in T_1(c_k)$ and $g$ is $L_N$-Lipschitz, we have again by~\eqref{eq:GC}
  \[
    \forall \xi \in \R, \quad
    \E\big[e^{\xi(g(\theta_k) - \E[g(\theta_k)])}\big] \le \exp\!\bigg(\frac{c_k\xi^2 }{2} L_N^2\bigg).
  \]
  Combining the last two inequalities thus yields
  \[
    \forall \xi \in \R, \quad
    \E\big[e^{\xi(f(\theta_{k+1}) - \E[f(\theta_{k+1})])}\big]
    \le \exp\!\bigg(\frac{\xi^2}{2}\big(L_N^2 c_k + K^2 C_\pi^{(1)}\big)\bigg)
  \]
  which is in turn equivalent to $\nu_{k+1} \in T_1(c_{k+1})$ with $c_{k+1}$ as in~\eqref{eq:ck}. Under $L_N < 1$, it follows that for all $k \ge 1$,
  \[
    c_k = L_N^{2k} NC_0^{(1)} + K^2 C_\pi^{(1)} \sum_{j=0}^{k-1} L_N^{2j}
    \le NC_0^{(1)} + \frac{K^2}{(1 - L_N^2)} C_\pi^{(1)} = C_N^{(1)},
  \]
  which is the claimed bound for $p = 1$.

  \bigskip

  \noindent\emph{Case $p = 2$.} Here, we adopt a different approach: we first establish a $T_2$ inequality for the joint law of the whole trajectory $(\theta_0, \dots, \theta_k) \in \mathcal{E}^{k+1}$ using a dependent tensorization theorem for contracting Markov chains, and then project onto the last coordinate to recover a $T_2$ inequality for $\nu_k$, with a constant independent of $k$.

  Fix $k \in \N$ and denote by $\nu^{(k)}$ the law of the trajectory $(\theta_0, \dots, \theta_k)$ on $\mathcal{E}^{k+1}$, equipped with the $\ell^2$ product distance
  \[
    \d_2\big((\theta_0, \dots, \theta_k), (\theta'_0, \dots, \theta'_k)\big)
    = \bigg(\sum_{j=0}^k \|\theta_j - \theta'_j\|^2_2\bigg)^{1/2}.
  \]
  We apply~\cite[Theorem~1.2]{Blower-Bolley} with $s = 2$ to the measure $\nu^{(k)}$ on $(\mathcal{E}^{k+1}, \d_2)$, and verify that assumptions (i) and (ii) there are satisfied.

  First, since $\mu_0 \in T_2\big(C_0^{(2)}\big)$, by tensorization we have $\nu_0 = \mu_0^{\otimes N} \in T_2\big(C_0^{(2)}\big)$. Moreover, using stability under Lipschitz pushforward and $\pi \in T_2\big(C_\pi^{(2)}\big)$, we get by~\cref{lem:Lipschitz}~\ref{Lipschitz:i} and~\eqref{eq:kernel} that $P(\theta, \cdot) \in T_2(C_P)$ with $C_P = C_\pi^{(2)} K^2 / N$ independent of $\theta$. Therefore, condition (i) in~\cite[Theorem~1.2]{Blower-Bolley} holds with constant $C_0^{(2)} \vee C_P$.

  Second, let $Z \sim \pi$ and consider the coupling $\big(\Phi(\theta,Z), \Phi(\theta',Z)\big)$ of $P(\theta,\cdot)$ and $P(\theta',\cdot)$. The contraction condition (ii) in~\cite[Theorem~1.2]{Blower-Bolley} is then a direct consequence of~\cref{lem:Lipschitz}~\ref{Lipschitz:ii} since
  \[
    \W_2^2\big(P(\theta,\cdot), P(\theta',\cdot)\big)
    \le \E\big[\|\Phi(\theta,Z) - \Phi(\theta',Z)\|_2^2\big]
    \le L_N^2\|\theta - \theta'\|_2^2.
  \]
  Under $L_N < 1$,~\cite[Theorem~1.2]{Blower-Bolley} gives $\nu^{(k)} \in T_2\big(C_N^{(2)}\big)$ on $(\mathcal{E}^{k+1}, \d_2)$ with constant $C_N^{(2)} = (1 - L_N)^{-2} \big(C_0^{(2)} \vee C_P\big)$ independent of $k$. Finally, the projection $(\theta_0, \dots, \theta_k) \mapsto \theta_k$ is $1$-Lipschitz from $(\mathcal{E}^{k+1}, \d_2)$ to $(\mathcal{E}, \|\cdot\|_2)$, therefore $\nu_k \in T_2\big(C_N^{(2)}\big)$, which is the claimed bound for $p = 2$.
\end{proof}

\subsection{Uniform bias decay}

We now derive a uniform-in-time bound on the bias term $\E[\Delta_t^{N,\circ}]$ appearing in~\cref{cor:GC}. To this end, we compare the SGD iterates to an auxiliary i.i.d.\ particle system driven by the mean-field dynamics, defined below.

Let us first formalize the mean-field limit by introducing the evolution equation governing the law of a typical particle in the infinite-width regime. We define the mean-field limit $(\bar\mu_t)_{t \ge 0}$ as the unique (deterministic) solution in $\mathcal{C}(\R_+, \mathcal{P}_1(\R^D))$ to the continuity equation
\begin{equation}\label{eq:MF}
  \partial_t \bar\mu_t + \alpha\,\nabla\cdot\big(G_\lambda(\cdot, \bar\mu_t)\,\bar\mu_t\big) = 0,\qquad 
  \bar\mu_0 = \mu_0,
\end{equation}
where $G_\lambda$ represents the \emph{mean-field drift}: for any $(\theta^i, \mu) \in \R^D \times \mathcal{P}(\R^D)$,
\begin{equation}\label{eq:MF-drift}
  G_\lambda(\theta^i, \mu) = \E_\pi \big[\big(y - \langle \sigma_*(\cdot, x), \mu \rangle\big) \nabla_{\theta^i}\sigma_*(\theta^i, x)\big] - \lambda\theta^i.
\end{equation} 
Equation~\eqref{eq:MF} is to be understood in the \emph{weak sense} (see, e.g.,~\cite[Definition~4.1]{Santambrogio}), namely for any smooth compactly supported test function $\varphi \in \mathcal{C}^1_c(\R^D)$,
\[
  \frac{d}{dt} \langle \varphi, \bar\mu_t \rangle
  = \alpha\,\big\langle \nabla\varphi \cdot G_\lambda(\cdot, \bar\mu_t), \bar\mu_t \big\rangle.
\]
Under~\ref{assump:A1}--\ref{assump:A3}, existence and uniqueness for~\eqref{eq:MF} in the penalized case follows from a straightforward adaptation of the arguments in~\cite[§2.3.1]{DGMN} (the regularization term adds a linear dissipative drift and the Lipschitz estimates in $\W_1$ used there remain unchanged).

A useful probabilistic tool for the analysis of~\eqref{eq:MF} is given by the associated \emph{nonlinear dynamics}: following the arguments of~\cite[Theorem~1.1]{Sznitmann}, there exists a unique (in law) trajectorial solution $(\bar\theta_t)_{t \ge 0}$ to
\begin{equation}\label{eq:nonlinear-dynamics}
  \dot{\bar\theta}_t = \alpha\,G_\lambda(\bar\theta_t, \rho_t),\qquad
  \bar\theta_0 \sim \mu_0,\qquad 
  \rho_t = \L(\bar\theta_t),
\end{equation}
and $(\rho_t)_{t \ge 0}$ satisfies~\eqref{eq:MF}. In addition, under~\ref{assump:A1},~\ref{assump:A2} and~\ref{assump:A5}, it is straightforward to show that $(\rho_t)_{t \ge 0} \in \mathcal{C}(\R_+, \mathcal{P}_1(\R^D))$ and that $t \ge 0 \mapsto \rho_t$ is in fact Lipschitz-continuous in $\W_1$. It follows by uniqueness of the solution to the mean-field PDE~\eqref{eq:MF} that $\rho_t = \bar\mu_t$ for all $t \ge 0$. Therefore, we introduce, as it is customary, the i.i.d.\ particle system $(\bar\theta^i_t)_{t \ge 0}$, $i \in \{1, \dots, N\}$, starting from the same points as~\eqref{eq:SGD} and evolving independently for $t \ge 0$ according to the nonlinear dynamics~\eqref{eq:nonlinear-dynamics}:
\begin{equation}\label{eq:MF-particles}
  \forall t \ge 0,\quad 
  \bar\theta^i_t = \theta^i_0 + \alpha \int_0^t G_\lambda(\bar\theta^i_s, \bar\mu_s)\,ds,
\end{equation}
so that $\L(\bar\theta^i_t) = \bar\mu_t$ for all $t \ge 0$. The associated empirical measure is denoted by
\[
  \bar\mu_t^N = \frac1N \sum_{i=1}^N \delta_{\bar\theta^i_t}.
\]

By triangular inequality, we decompose the bias into
\begin{equation}\label{eq:split}
  \E[\Delta_t^N] \le 
  \E\big[ \W_1(\mu_t^N, \bar\mu_t^N) \big]
  + \E\big[ \W_1(\bar\mu_t^N, \bar\mu_t) \big],
\end{equation}
and similarly for $\E[\Delta_t^{N,\varphi}]$ and $\E[\Delta_t^{N,\s}]$. The first term in~\eqref{eq:split} captures a \emph{dynamic} error between the SGD particles and their mean-field counterparts. It is controlled uniformly over $t \ge 0$ by a standard \emph{propagation-of-chaos} argument~\cite{Sznitmann}, based on a synchronous coupling.

\begin{proposition}\label{prop:PoC}
  Assume~\ref{assump:A1}--\ref{assump:A3}, \ref{assump:A5}, and $\lambda > \lambda_\star$, $N \ge N_\star$ (see~\eqref{eq:constants}). Then there exists $C > 0$ (independent of $N$ and $t$) such that
  \[
    \sup_{t \ge 0} \E\big[\W_1(\mu_t^N, \bar\mu_t^N)\big] \le \frac{C}{\sqrt{N}}.
  \]
\end{proposition}

Recalling $|\langle \varphi, \mu \rangle - \langle \varphi, \nu \rangle| \le \W_1(\mu,\nu)$ for any $1$-Lipschitz $\varphi$ and $\SW_1 \le \W_1$,~\cref{prop:PoC} directly implies Corollary~\ref{cor:PoC} below.

\begin{corollary}\label{cor:PoC}
  Under the assumptions of~\cref{prop:PoC}, there exists $C > 0$ (independent of $N$ and $t$) such that
  \[
    \sup_{t \ge 0} \E\big[\big|\langle \varphi,\mu_t^N \rangle - \langle \varphi,\bar\mu_t^N \rangle\big|\big] \le \frac{C}{\sqrt{N}},\qquad
    \sup_{t \ge 0} \E\big[\SW_1(\mu_t^N, \bar\mu_t^N)\big] \le \frac{C}{\sqrt{N}}.
  \]
\end{corollary}

The second term in~\eqref{eq:split} is a purely \emph{static} i.i.d.\ sampling error of empirical measures, and is controlled by the optimal $\W_1$ rate of~\cite[Theorem~1]{Fournier-Guillin}, which we recall in the following proposition.

\begin{proposition}\label{prop:FG}
  Assume~\ref{assump:A1},~\ref{assump:A2} and~\ref{assump:A5}. Then there exists $C > 0$ (independent of $N$ and $t$) such that
  \[
    \sup_{t \ge 0} \E\big[\W_1(\bar\mu_t^N, \bar\mu_t)\big] \le CN^{-1/(1 \vee D)}.
  \]
\end{proposition}

We also require static i.i.d.\ sampling rates for the empirical measure $\bar\mu_t^N$ when tested against a test function $\varphi$ and through $\SW_1$. Unlike for $\W_1$, we expect these two quantities to admit a uniform $N^{-1/2}$ rate independent of the parameter dimension $D$, since the first is simply a scalar sample average and the second reduces to one-dimensional Wasserstein distances through projections. This intuition is formalized in the next proposition.

\begin{proposition}\label{prop:FG2}
  Assume~\ref{assump:A5}. Then there exists $C > 0$ (independent of $N$ and $t$) such that
  \[
    \sup_{t \ge 0} \E\big[\big|\langle \varphi,\bar\mu_t^N\rangle-\langle \varphi,\bar\mu_t\rangle\big|\big]\ \le\ \frac{C}{\sqrt N},\qquad
    \sup_{t \ge 0} \E\big[\SW_1(\bar\mu_t^N,\bar\mu_t)\big]\ \le\ \frac{C}{\sqrt N}.
  \]
\end{proposition}

Combining~\eqref{eq:split}, Propositions~\ref{prop:PoC}--\ref{prop:FG2} and~\cref{cor:PoC} yields~\cref{prop:bias} (note that condition $L_N < 1$ in~\cref{prop:bias} implies $\lambda > \lambda_\star$ in~\cref{prop:PoC}). The proofs of Propositions~\ref{prop:PoC},~\ref{prop:FG} and~\ref{prop:FG2} are carried out in the next two sections.

\subsubsection{Proof of Propositions~\ref{prop:FG}-\ref{prop:FG2}}

We begin by establishing a uniform-in-time moment bound for the i.i.d.\ mean-field particles $\bar\theta^i_t$. It will be invoked repeatedly throughout the proofs of Propositions~\ref{prop:PoC},~\ref{prop:FG}, and~\ref{prop:FG2}.

\begin{lemma}\label{lem:moments}
  Assume~\ref{assump:A1}--\ref{assump:A2}, and $\E\big[|\theta^1_0|^q\big] < \infty$ for some $q > 1$. Then there exists $C_q > 0$ (independent of $N$ and $t$) such that
  \[
    \sup_{t \ge 0} \E\big[|\bar\theta^1_t|^q\big] \le C_q < \infty.
  \]
\end{lemma}

\begin{proof}
  Fix $q > 1$ such that $\E\big[|\theta^1_0|^q\big] < \infty$, and let $f_q(t) = \E\big[|\bar\theta^1_t|^q\big]$, $t \ge 0$. For $q > 1$, the map $N_q : w \mapsto |w|^q$ is continuously differentiable on $\R^D$, with $\nabla N_q(w) = q|w|^{q-2}w$ for $w \ne 0$ and $\nabla N_q(0) = 0$. Therefore, using~\eqref{eq:nonlinear-dynamics} and the chain rule, we have
  \begin{align*}
    f_q'(t)
    &= \E\big[\nabla N_q(\bar\theta^1_t) \cdot \dot{\bar\theta}^1_t\big]\\
    &= q\alpha\,\E\big[|\bar\theta^1_t|^{q-2} \bar\theta^1_t \cdot G_\lambda(\bar\theta^1_t, \bar\mu_t) \big]\\
    &\le - q\alpha\lambda\,\E\big[|\bar\theta^1_t|^q\big] + q\alpha C\,\E\big[|\bar\theta^1_t|^{q-1}\big]
  \end{align*}
  with $C = (A+B)M$ in the last inequality by~\ref{assump:A1}--\ref{assump:A2}. By Young's inequality with exponents $q$ and $q' = q / (q-1)$, for any $\varepsilon > 0$,
  \[
    C|\bar\theta^1_t|^{q-1} 
    \le \frac{\varepsilon^{q'}}{q'} |\bar\theta^1_t|^{q} + \frac{C^q}{q \varepsilon^q}.
  \]
  Choosing $\varepsilon^{q'} = \lambda$ and injecting into the previous differential inequality, we get
  \[
    f_q'(t) \le -\alpha\lambda f_q(t) + \alpha \frac{C^q}{\lambda^{q-1}}.
  \]
  A Grönwall estimate then gives, for all $t \ge 0$,
  \[
    f_q(t) \le e^{-\alpha\lambda t} f_q(0) + \frac{C^q}{\lambda^q}(1 - e^{-\alpha\lambda t}) \le \E\big[|\theta^1_0|^q\big] + \frac{C^q}{\lambda^q} = C_q < \infty,
  \]
  which concludes the proof.
\end{proof}

\begin{proof}[Proof of~\cref{prop:FG}]
  Let $q > 2$ be given by~\ref{assump:A5}. By~\cref{lem:moments}, there exists $C_q > 0$ (independent of $t$ and $N$) such that
  \[
    \sup_{t\ge 0}\E\big[|\bar\theta_t^1|^q\big] \le C_q.
  \]
  Therefore, applying~\cite[Theorem~1]{Fournier-Guillin} with $p=1$ and exponent $q > 2$ to the i.i.d.\ sample $(\bar\theta_t^1,\dots,\bar\theta_t^N)$ of law $\bar\mu_t$ yields, for all $t \ge 0$,
  \[
    \E\big[\W_1(\bar\mu_t^N,\bar\mu_t)\big]
    \le C\,\big(\E[|\bar\theta_t^1|^q]\big)^{1/q}\,N^{-1/(1 \vee D)},
  \]
  where $C$ depends only on $D$ and $q$. The result follows by taking the supremum over $t \ge 0$.
\end{proof}

\begin{proof}[Proof of~\cref{prop:FG2}]
  Let $t \ge 0$ and set $U_t^i = \varphi(\bar\theta_t^i)$, which are i.i.d.\ with $\E[U_t^1] = \langle \varphi, \bar\mu_t \rangle$. By the Cauchy--Schwarz inequality, 
  \[
    \E\big[\big|\langle \varphi,\bar\mu_t^N\rangle-\langle \varphi,\bar\mu_t\rangle\big|\big]
    = \E\Big[\Big|\frac1N\sum_{i=1}^N (U_t^i-\E[U_t^1])\Big|\Big]
    \le \frac{\sqrt{\Var(U_t^1)}}{\sqrt{N}}.
  \]
  Since $\varphi$ is $1$-Lipschitz, $\Var(U_t^1) \le \E\big[|U_t^1-\varphi(0)|^2\big] \le \E\big[|\bar\theta_t^1|^2\big]$. By~\ref{assump:A5}, $\E[|\theta_0^1|^2] < \infty$. We may thus apply the uniform moment bound of~\cref{lem:moments} with $q = 2$ and take the supremum over $t\ge 0$ in the previous display to obtain
  \[
    \sup_{t \ge 0} \E\big[\big|\langle \varphi,\bar\mu_t^N\rangle-\langle \varphi,\bar\mu_t\rangle\big|\big] \le \frac{C}{\sqrt{N}}.
  \] 
  
  Let us now turn to the analogous bound for the sliced-Wasserstein distance. By Fubini's theorem,
  \begin{equation}\label{eq:SW}
     \E\big[\SW_1(\bar\mu_t^N, \bar\mu_t)\big]
    = \int_{\S^{D-1}} \E\big[ \W_1\big({P_u}_\#\bar\mu_t^N, {P_u}_\#\bar\mu_t\big)\big]\,\zeta(du).
  \end{equation}
  For each $u\in\S^{D-1}$, ${P_u}_\#\bar\mu_t^N$ is the empirical measure of $N$ i.i.d.\ samples with common law ${P_u}_\#\bar\mu_t$ on $\R$. By~\ref{assump:A5} and Lemma~\ref{lem:moments}, there exists $q > 2$ and $C_q > 0$ (independent of $D$, $N$ and $t$) such that
  \[
    \sup_{t \ge 0} \sup_{u \in \S^{D-1}} \int_{\R} |w|^q\,{P_u}_\#\bar\mu_t(dw)
    = \sup_{t \ge 0} \sup_{u \in \S^{D-1}} \E\big[|\langle u, \bar\theta^1_t \rangle|^q\big]
    \le \sup_{t \ge 0} \ \E\big[|\bar\theta^1_t|^q\big]
    \le C_q.
  \]
  Therefore, applying~\cite[Theorem~1]{Fournier-Guillin} in dimension $D = 1$ with $p = 1$ gives a constant $C > 0$ (depending only on $q$) such that for all $t \ge 0$ and for all $u \in \S^{D-1}$,
  \[
    \E\big[\W_1\big({P_u}_\#\bar\mu_t^N, {P_u}_\#\bar\mu_t\big)\big] \le \frac{C}{\sqrt{N}}.
  \]
  Integrating over $u \sim \zeta$ and taking the supremum over $t \ge 0$ thus yields the desired bound.
\end{proof}

\subsubsection{Proof of Proposition~\ref{prop:PoC}}

Although the network's parameters are initialized i.i.d., this independence is generally lost under the SGD dynamics. However, it is easy to see that \emph{exchangeability} of the system is preserved at all times, i.e. the law of the parameters is left invariant under permutations of their coordinates. Indeed, this is a direct consequence of the \emph{permutation-equivariance} of the update map $\Phi$: for any permutation $\sigma$ of $\{1, \dots, N\}$,
\begin{equation}\label{eq:permutation-equivariance}
  \forall \theta \in \mathcal{E},\quad 
  \Phi(\theta^\sigma, \cdot) = \Phi(\theta, \cdot)^\sigma
\end{equation}
where $\theta^\sigma = (\theta^{\sigma(1)}, \dots, \theta^{\sigma(N)})$. In addition, even though the SGD and mean-field systems are coupled through the same initialization (so they are \emph{a priori} not independent), the resulting family of \emph{paired} particles is still exchangeable. This is made precise in the following lemma.

\begin{lemma}\label{lem:exchangeability}
  For every $t \ge 0$, the law of $\big(\theta_{\lfloor Nt \rfloor}, \bar\theta_t\big)$ is exchangeable.
\end{lemma}

\begin{proof}
Fix a permutation $\sigma$ of $\{1, \dots, N\}$, and let $t \ge 0$ and $k = \lfloor Nt \rfloor$. By definition~\eqref{eq:MF-particles}, there exists a deterministic map $\Psi_t : \R^D \to \R^D$ (independent of $i$) such that for every $i \in \{1, \dots, N\}$, $\bar\theta^i_t = \Psi_t(\theta^i_0)$. In particular,
\[
  \bar\theta_t^\sigma
  = \big(\bar\theta^{\sigma(1)}_t, \dots, \bar\theta^{\sigma(N)}_t\big)
  = \big(\Psi_t(\theta^{\sigma(1)}_0), \dots, \Psi_t(\theta^{\sigma(N)}_0)\big)
  = \Psi_t(\theta_0^\sigma),
\]
where by abuse of notation we extend $\Psi_t : \mathcal{E} \to \mathcal{E}$ componentwise.

On the other hand, by iterating~\eqref{eq:SGD-update} we may write $\theta_k = \Phi^{(k)}(\theta_0, z_1, \dots, z_k)$, where $\Phi^{(k)}: \mathcal{E} \times \mathcal{Z}^k \to \mathcal{E}$ is defined recursively as $\Phi^{(0)}(\theta) = \theta$ and for $j \in \N$,
\[
  \Phi^{(j+1)}(\theta, z_1, \dots, z_{j+1}) = \Phi(\Phi^{(j)}(\theta, z_1, \dots,z_j), z_{j+1}).
\] 
Therefore, by iterated applications of~\eqref{eq:permutation-equivariance}, it follows that for all $\theta \in \mathcal{E}$ and $(z_1, \dots, z_k) \in \mathcal{Z}^k$,
\[
  \Phi^{(k)}(\theta^\sigma, z_1, \dots, z_k) = \Phi^{(k)}(\theta, z_1, \dots, z_k)^\sigma.
\]

Hence the maps $\Psi_t$ and $\Phi^{(k)}$ are also permutation-equivariant. Using exchangeability of the initial parameter vector $\theta_0$ and independence between $\theta_0$ and the datapoints $(z_1, \dots, z_k)$, we thus have
\[
  (\theta_k, \bar\theta_t)^\sigma
  = \big(\Phi^{(k)}(\theta_0^\sigma, z_1, \dots, z_k), \Psi_t(\theta_0^\sigma)\big)
  \stackrel{d}{=} \big(\Phi^{(k)}(\theta_0, z_1, \dots, z_k), \Psi_t(\theta_0)\big)
  = (\theta_k, \bar\theta_t)
\]
which proves exchangeability of the synchronously coupled particles.
\end{proof}

Using exchangeability and the same synchronous coupling as in~\cref{lem:exchangeability}, we have
\begin{equation}\label{eq:empirical-to-one}
  \E\big[ \W_1(\mu_t^N, \bar\mu_t^N) \big]
  \le \frac1N \sum_{i=1}^N \E\big[|\theta^i_{\lfloor Nt \rfloor} - \bar\theta^i_t|\big]
  = \E\big[|\theta^1_{\lfloor Nt \rfloor} - \bar\theta^1_t|\big].
\end{equation}
Thus exchangeability reduces the $\W_1$ discrepancy to the expected deviation of a single tagged particle.

\begin{proof}[Proof of~\cref{prop:PoC}]
By the above argument, it suffices to bound $\E\big[|\theta^1_{\lfloor Nt \rfloor} - \bar\theta^1_t|\big]$ uniformly in $t \ge 0$. For $t \ge 0$, set $k = \lfloor Nt \rfloor$ and $t_k = k/N$ (so that $t \in \big[t_k, t_{k+1}\big)$). Define for $i \in \{1, \dots, N\}$, 
\[
  \delta^i_k = \theta^i_k - \bar\theta^i_{t_k},\qquad
  \delta_k = \delta_k^1,\qquad 
  D_t = \bar\theta^1_{t_k} - \bar\theta^1_t,  
\]
so that 
\begin{equation}\label{eq:SGD-MF}
  |\theta^1_{\lfloor Nt \rfloor} - \bar\theta^1_t| \le |\delta_k| + |D_t|.  
\end{equation}
By~\eqref{eq:MF-particles},~\ref{assump:A1} and~\ref{assump:A2}, for all $t \ge 0$,
\[
  \E\big[|D_t|\big] 
  \le \alpha \int_{t_k}^{t} \E\big[|G_\lambda(\bar\theta^1_s,\bar\mu_s)|\big]\,ds
  \le \gamma \big((A+B)M + \lambda\sup_{t \ge 0} \E\big[|\bar\theta^1_t|\big]\big)
\]
where we have used in the last inequality $t - t_k \le 1/N$ and $\gamma = \alpha / N$. By~\cref{lem:moments}, under the assumption~\ref{assump:A5}, we have $\sup_{t \ge 0} \E\big[|\bar\theta^1_t|\big] \le 1 + \sup_{t \ge 0} \E\big[|\bar\theta^1_t|^2\big] < \infty$. Hence, the discretization error $D_t$ is uniformly bounded as
\begin{equation}\label{eq:D}
  \sup_{t \ge 0} \E\big[|D_t|\big] \le \frac{C}{N}.
\end{equation}
It remains to control $\E\big[|\delta_k|\big]$ uniformly over $k \in \N$. By~\eqref{eq:MF-particles}, we have for all $k \in \N$,
\[
  \bar\theta^1_{t_{k+1}}
  = \bar\theta^1_{t_k} + \alpha \int_{t_k}^{t_{k+1}} G_\lambda(\bar\theta^1_s, \bar\mu_s)\,ds,
\]
hence, adding and subtracting $\gamma G_\lambda(\bar\theta^1_{t_k}, \bar\mu_{t_k})$, 
\begin{equation}\label{eq:delta}
  \delta_{k+1} = \delta_k + \gamma \left(F^1_\lambda(\theta_k, z_{k+1}) - G_\lambda(\bar\theta^1_{t_k}, \bar\mu_{t_k})\right) + E_{k+1}
\end{equation}
where 
\[
  E_{k+1} = \gamma G_\lambda(\bar\theta^1_{t_k}, \bar\mu_{t_k}) - \alpha \int_{t_k}^{t_{k+1}} G_\lambda(\bar\theta^1_s, \bar\mu_s)\,ds.
\] 
We decompose as in~\cite[Lemma~7.2]{MMN}
\begin{align*}
  F^1_\lambda(\theta_k, z_{k+1}) - G_\lambda(\bar\theta^1_{t_k}, \bar\mu_{t_k})
  &= F^1_\lambda(\theta_k, z_{k+1}) - G_\lambda(\theta^1_k, \mu^N_{t_k})\\
  &+ G_\lambda(\theta^1_k, \mu^N_{t_k}) - G_\lambda(\theta^1_k, \bar\mu_{t_k})\\
  &+ G_\lambda(\theta^1_k, \bar\mu_{t_k}) - G_\lambda(\bar\theta^1_{t_k}, \bar\mu_{t_k})\\
  &\equiv A_k + B_k + C_k.
\end{align*}
Set $m_k = \E\big[|\delta_k|^2\big]$. Note that $m_0 = 0$ since $\delta_0 = 0$ under the synchronous coupling. Using the recursion for $\delta_{k+1}$~\eqref{eq:delta} and the above decomposition, expanding $|\delta_{k+1}|^2$ and taking expectations yields
\begin{align*}
  m_{k+1} \le m_k 
  &+ \gamma^2\,\E\big[|A_k + B_k + C_k|^2\big] + 2\gamma\,\E\big[\delta_k \cdot (A_k + B_k + C_k)\big]\\
  &+ \E\big[|E_{k+1}|^2\big] + 2\,\E\big[\delta_k \cdot E_{k+1}\big] + 2\gamma\,\E\big[(A_k + B_k + C_k) \cdot E_{k+1}\big]. 
\end{align*}
We study each term on the right-hand side in the following technical lemma. 

\begin{lemma}\label{lem:ABCE}
  Assume~\ref{assump:A1}--\ref{assump:A3} and~\ref{assump:A5}. Then the following hold for all $k \in \N$:
  \begin{enumerate}[label={\normalfont(\roman*)}]
    \item\label{ABCE:i}
    Let $(\F_k)_{k \in \N}$ be the filtration given by $\F_k = \sigma(\theta_0, z_1, \dots, z_k)$. \\
    Then $\E[A_k \mid \F_k] = 0$ and $\E\big[|A_k|^2\big] \le C$.

    \item\label{ABCE:ii}
    $\E\big[\delta_k \cdot B_k\big] \le M^2\,m_k + CN^{-1}$, and\ 
    $\E\big[|B_k|^2\big] \le 2M^4\,m_k + 2M^2B^2N^{-1}$.

    \item\label{ABCE:iii}
    $\E\big[\delta_k \cdot C_k\big] \le -\big(\lambda - (A+B)L_\theta\big)\,m_k$, and\ 
    $\E\big[|C_k|^2\big] \le \big(\lambda+(A+B)L_\theta\big)^2\,m_k$.

    \item\label{ABCE:iv}
    $\E\big[|E_{k+1}|\big] \le CN^{-2}$,\ 
    $\E\big[|E_{k+1}|^2\big] \le CN^{-4}$, and\ $\E\big[\delta_k \cdot E_{k+1}\big] \le \frac14 \gamma(\lambda - \lambda_\star)\,m_k + C\gamma^2$.
  \end{enumerate}
\end{lemma}

\smallskip

The proof of~\cref{lem:ABCE} is postponed to the end of this section. We now conclude.

Since $\delta_k, B_k, C_k, E_{k+1}$ are $\F_k$-measurable and $\E\big[A_k \mid \F_k\big]=0$ by~\cref{lem:ABCE}~\ref{ABCE:i}, all cross-terms involving $A_k$ vanish:
\[
    \E\big[\delta_k \cdot A_k\big]
  = \E\big[A_k \cdot B_k\big]
  = \E\big[A_k \cdot C_k\big]
  = \E\big[A_k \cdot E_{k+1}\big]
  = 0.
\]
Hence, using $2\gamma\,(B_k + C_k) \cdot E_{k+1} \le \gamma^2|B_k + C_k|^2 + |E_{k+1}|^2$ and $|B_k + C_k|^2 \le 2|B_k|^2 + 2|C_k|^2$, we get
\begin{align*}
  m_{k+1}
  \le m_k
    &+ \gamma^2\,\E\big[|A_k|^2 \big]
    + 4\gamma^2\,\E\big[|B_k|^2 + |C_k|^2\big]
    + 2\gamma\,\E\big[\delta_k \cdot B_k\big]
    + 2\gamma\,\E\big[\delta_k \cdot C_k\big]\\
    &+ 2\,\E\big[|E_{k+1}|^2\big]
    + 2\,\E\big[\delta_k \cdot E_{k+1}\big].
\end{align*}
Using the bounds of~\cref{lem:ABCE}, we obtain after rearranging terms
\[
  m_{k+1}
  \le m_k
    - \frac32 \gamma(\lambda-\lambda_\star)\,m_k
    + C_\star\gamma^2\,m_k
    + C\gamma^2
    + \frac{C}{N^3}
\]
with $C_\star = 8M^4 + 4\big(\lambda+(A+B)L_\theta\big)^2$ (see~\eqref{eq:constants}). Since $N \ge N_\star$, we have $C_\star \gamma^2 \le \frac14 \gamma(\lambda-\lambda_\star)$, and hence
\[
  m_{k+1} \le \big(1-\gamma(\lambda-\lambda_\star)\big)m_k + C\gamma^2 + \frac{C}{N^3}.
\]
Since $\lambda > \lambda_\star$, a discrete Grönwall argument and $m_0 = 0$ give
\begin{equation}\label{eq:supm}
  \sup_{k \in \N} m_k \le \frac{C\gamma^2}{\gamma (\lambda - \lambda_\star)} \le \frac{C}{N}.
\end{equation}
Finally, combining~\eqref{eq:D} with~\eqref{eq:supm} in~\eqref{eq:SGD-MF} yields
\[
  \sup_{t \ge 0} \E\big[|\theta^1_{\lfloor Nt\rfloor}-\bar\theta^1_t|\big] 
  \le \sup_{k \in \N} \sqrt{m_k} + \frac{C}{N}
  \le \frac{C}{\sqrt{N}}.
\]
Recalling~\eqref{eq:empirical-to-one}, this concludes the proof of~\cref{prop:PoC}.
\end{proof}

\begin{proof}[Proof of~\cref{lem:ABCE}] Fix $k \in \N$ and recall $t_k = k/N$, so that $\mu^N_{t_k}$ is the empirical measure of the SGD parameters at step $k$.

\bigskip

\noindent\ref{ABCE:i} By definitions~\eqref{eq:SGD-drift} and~\eqref{eq:MF-drift}, since $z_{k+1}$ is independent of $\F_k$ and $\theta_k$ is $\F_k$-measurable,
\[
  \E\big[ F_\lambda^1(\theta_k, z_{k+1}) \mid \F_k \big] = G_\lambda(\theta^1_k, \mu^N_{t_k}),
\]
hence $\E\big[A_k \mid \F_k \big] = 0$. Moreover, by~\ref{assump:A1}--\ref{assump:A2},
\[
  |A_k|
  = \big|\big(y_{k+1} - \hat{y}_{\theta_k}(x_{k+1})\big) \nabla_{\theta^1} \sigma_*(\theta^1_k, x_{k+1})  
    - \E_\pi \big[(y - \hat{y}_{\theta_k}(x)) \nabla_{\theta^1} \sigma_*(\theta^1_k, x)\big]\big|\\
  \le 2(A+B)M,  
\]
so that $\E\big[|A_k|^2\big] \le 4(A+B)^2 M^2$.

\bigskip

\noindent\ref{ABCE:ii} Write $B_k = B_{1, k} + B_{2, k}$ with
\[
  B_{1, k} = G_\lambda(\theta^1_k, \mu^N_{t_k}) - G_\lambda(\theta^1_k, \bar\mu^N_{t_k})\qquad
  B_{2, k} = G_\lambda(\theta^1_k, \bar\mu^N_{t_k}) - G_\lambda(\theta^1_k, \bar\mu_{t_k}).
\]
Using $\|\nabla_{(\theta^i, x)}\sigma_*\|_\infty \le M$ and the $M$-Lipschitzness of the map $\theta^i \in \R^D \mapsto \sigma_*(\theta^i, x)$ from~\ref{assump:A1}, we have
\[
  |B_{1, k}|
  \le M\,\E_\pi\big[|\hat{y}_{\theta_k}(x) - \langle \sigma_*(\cdot, x), \bar\mu^N_{t_k}\rangle|\big]
  \le \frac{M}{N} \sum_{i=1}^N \E_\pi\big[ | \sigma_*(\theta^i_k,x) - \sigma_*(\bar\theta^i_{t_k},x) | \big]
  \le \frac{M^2}{N} \sum_{i=1}^N |\delta^i_k|.
\]
Hence, by exchangeability (see~\cref{lem:exchangeability}),
\[
  \E\big[|B_{1, k}|^2\big]
  \le M^4\,\E\bigg[\bigg(\frac1N \sum_{i=1}^N |\delta_k^i|\bigg)^2\bigg]
  \le \frac{M^4}{N} \sum_{i=1}^N \E\big[|\delta_k^i|^2\big]
  = M^4\,m_k.
\]
Using again $\|\nabla_{(\theta^i, x)}\sigma_*\|_\infty \le M$ and Jensen's inequality, we have
\[
  |B_{2, k}|^2
  \le M^2\,\E_\pi\bigg[\Big|\frac1N \sum_{i=1}^N \sigma_*(\bar\theta^i_{t_k},x) - \E\big[\sigma_*(\bar\theta^1_{t_k},x)\big]\Big|^2\bigg] = \frac{M^2}{N} \E_\pi\big[\Var\big(\sigma_*(\bar\theta^1_{t_k},x)\big)\big]
  \le \frac{M^2B^2}{N}.
\]
Combining the two bounds for $B_{1, k}$ and $B_{2, k}$, it folllows that
\[
  \E\big[|B_k|^2\big]
  \le 2\,\E\big[|B_{1, k}|^2\big] + 2\,\E\big[|B_{2, k}|^2\big]
  \le 2M^4\,m_k + \frac{2M^2B^2}{N}.
\]
Hence, by the Cauchy--Schwarz inequality,
\[
  \E\big[\delta_k \cdot B_k\big]
  \le \sqrt{m_k}\,\sqrt{\E\big[|B_k|^2\big]}
  \le \sqrt{C}\,m_k + \sqrt{\frac{C}{N}}\sqrt{m_k}.
\]
Finally, applying Young's inequality
\begin{equation}\label{eq:Young}
  \forall a, b \in \R,\ 
  \forall \varepsilon > 0,\quad 
  ab \le \frac{\varepsilon}{2} a^2 + \frac{1}{2\varepsilon} b^2  
\end{equation}
with $a = \sqrt{m_k}$, $b = \sqrt{C/N}$ and $\varepsilon = 2M^2$, we obtain
\[
  \E\big[\delta_k \cdot B_k\big] \le M^2\,m_k + \frac{C}{N}.
\]

\bigskip

\noindent\ref{ABCE:iii} By definition~\eqref{eq:MF-drift}, 
\[
  C_k
  = \E_\pi\big[\big(y - \langle \sigma_*(\cdot, x), \bar\mu^N_{t_k}\rangle\big)\big(\nabla_{\theta^1}\sigma_*(\theta_k^1, x)-\nabla_{\theta^1}\sigma_*(\bar\theta^1_{t_k}, x)\big)\big]
    - \lambda \delta_k.
\]
Using the Cauchy--Schwarz inequality and~\ref{assump:A3}, we have
\[
  \big|
    \E_\pi\big[ 
      \big(y - \langle \sigma_*(\cdot, x), \bar\mu^N_{t_k}\rangle\big) \big(\nabla_{\theta^1}\sigma_*(\theta_k^1,x) - \nabla_{\theta^1}\sigma_*(\bar\theta^1_{t_k},x)\big) 
    \big]
  \big|
  \le (A+B)L_\theta\,|\delta_k|.
\]
Hence
\begin{align*}
  \E\big[\delta_k \cdot C_k\big]
  &\le -\big(\lambda-(A+B)L_\theta\big)\,m_k,\\
  \E\big[|C_k|^2\big]
  &\le \big(\lambda + (A+B)L_\theta\big)^2\,m_k.
\end{align*}

\bigskip

\noindent\ref{ABCE:iv} Rewrite $E_{k+1}$ as
\[
  E_{k+1}
  = \alpha \int_{t_k}^{t_{k+1}} \big(
      G_\lambda(\bar\theta^1_{t_k}, \bar\mu_{t_k})
      - G_\lambda(\bar\theta^1_s, \bar\mu_s) 
    \big)\,ds.
\]
Using the Lipschitz bounds
\begin{align*}
  |G_\lambda(\theta, \mu) - G_\lambda(\theta', \mu)|
  &\le \big(\lambda + (A+B)L_\theta\big)\,|\theta - \theta'|,\\
  |G_\lambda(\theta, \mu) - G_\lambda(\theta, \nu)|
  &\le M^2\,\W_1(\mu, \nu),
\end{align*}
for all $\theta, \theta' \in \mathcal{E}$ and $\mu, \nu \in \mathcal{P}(\R^D)$, together with the fact that $s \mapsto \bar\theta^1_s$ is Lipschitz in $L^1$ and $s \mapsto \bar\mu_s$ is Lipschitz in $\W_1$ (with constants controlled by $(A+B)M + \lambda\,\sup_{t \ge 0} \E\big[|\bar\theta^1_t|\big] < \infty$ under~\ref{assump:A5}),
we obtain
\[
  \E\big[|G_\lambda(\bar\theta^1_{t_k},\bar\mu_{t_k}) - G_\lambda(\bar\theta^1_s,\bar\mu_s)|\big]
  \le C\,|s - t_k|.
\]
Therefore,
\[
  \E\big[|E_{k+1}|\big]
  \le \alpha \int_{t_k}^{t_{k+1}} C(s-t_k)\,ds
  \le \frac{C}{N^2},
\]
and similarly,
\[
  \E\big[|E_{k+1}|^2\big] \le \frac{C}{N^4}.
\]
Finally, by Young's inequality~\eqref{eq:Young} with $a = |\delta_k|$, $b = |E_{k+1}|$ and $\varepsilon = \frac12 \gamma (\lambda - \lambda_\star)$, it holds
\[
  \E\big[\delta_k \cdot E_{k+1}\big]
  \le \frac14 \gamma(\lambda-\lambda_\star)\, m_k
    + \frac{C}{\gamma(\lambda-\lambda_\star)N^4}\\
  \le \frac14 \gamma(\lambda-\lambda_\star)\,m_k + C\gamma^2.
\]
This concludes the proof.
\end{proof}

\section{Extension to unbounded activations}\label{sec:extension}

In this section, we extend our previous results to activations that are no longer uniformly bounded. To this end, we drop assumption~\ref{assump:A1} and replace it with the following localization assumptions~\ref{assump:B1}--\ref{assump:B3}. We show that under these conditions (and a suitable choice of $\lambda$), the SGD iterates remain uniformly bounded \emph{for all times} and thus evolve within a deterministic compact set. Once localization is established, the analysis reduces to the bounded case on that compact set.

\subsection{Localization assumptions}

In what follows, work under the following three assumptions in place of~\ref{assump:A1}.

\begin{enumerate}[label={(B\arabic*)}, font=\bfseries]
  \item\label{assump:B1} There exist a compact set $K_x \subset \R^d$ and radius $R_0 > 0$ such that
  \[
    \supp\pi \subset K_x \times [-A, A],\qquad
    \supp\mu_0 \subset \overline{B}(0, R_0) \subset \R^D.
  \]
  \item\label{assump:B2} There exist $b, c \ge 0$ such that for all $(\theta^i, x) \in \R^D \times K_x$, $|\sigma_*(\theta^i, x)| \le b + c|\theta^i|$.
  \item\label{assump:B3} The activation function $\sigma_* \in \mathcal{C}^1(\R^D \times \R^d)$ is continuously differentiable, with uniformly bounded gradient on $\R^D \times K_x$: 
  \[
    \forall (\theta^i, x) \in \R^D \times K_x, \quad 
    |\nabla_{(\theta^i, x)}\sigma_*(\theta^i, x)| \le M.
  \]
\end{enumerate}

Note that~\ref{assump:B1} implies~\ref{assump:A2}. Moreover, by Hoeffding's inequality,~\ref{assump:B1} guarantees sub-Gaussian Laplace bounds for Lipschitz observables, which is equivalent to a $T_1$ inequality. However, this condition alone is not sufficient to yield a $T_2$ inequality in general. The smoothness condition in~\ref{assump:B3} is purely technical. In particular, it excludes non-smooth activations such as ReLU; one may alternatively work with smoothed approximations (e.g.\ softplus) or piecewise-smooth variants (e.g.\ leaky-ReLU) without affecting the localization argument below.

\subsection{A uniform localization bound}

We will use the following constants:
\[
  a_{\infty} = R_0 \vee \frac{M(A+b)}{\lambda - Mc},\qquad
  R_{\infty} = R_0 \vee \frac{M(A+b+ca_{\infty})}{\lambda}.
\]

\begin{lemma}\label{lem:localization}
Assume \ref{assump:B1}--\ref{assump:B3}, and $\lambda > Mc$, $\gamma\lambda \le 1$. Then, almost surely,
\[
  \sup_{k \in \N} \max_{1 \le i \le N} |\theta_k^i| \le R_{\infty}.
\]
\end{lemma}

\begin{proof}
By~\eqref{eq:SGD} and \ref{assump:B1}--\ref{assump:B3}, for all $k \in \N$ and $i \in \{1, \dots, N\}$, we have
\[
  |\theta_{k+1}^i|
  \le (1 - \gamma\lambda)|\theta_k^i| + \gamma M|y_{k+1} - \hat{y}_{\theta_k}(x_{k+1})|
  \le (1 - \gamma\lambda)|\theta_k^i| + \gamma M(A + b + ca_k)
\]
where $a_k = \displaystyle\frac1N \sum_{i=1}^N |\theta_k^i|$. Averaging over $i$, we get
\begin{equation}\label{eq:ak}
  a_{k+1} \le \big(1-\gamma(\lambda-Mc)\big)\,a_k + \gamma M(A+b).
\end{equation}
Under $\gamma\lambda \le 1$ and $\lambda > Mc$, we have $0 \le 1-\gamma\lambda \le 1$ and $\gamma(\lambda - Mc) \le \gamma\lambda \le 1$, so that $1 - \gamma(\lambda-Mc) \in [0, 1)$. It follows by iteration of~\eqref{eq:ak} that
\[
  \sup_{k \in \N} a_k \le a_0 \vee \frac{M(A+b)}{\lambda-Mc} \le a_\infty
\]
where we used $a_0 \le \max\limits_{1 \le i \le N} |\theta_0^i| \le R_0$ from~\ref{assump:B1}. Plugging this bound back into~\eqref{eq:ak} yields
\[
  |\theta_{k+1}^i| \le (1-\gamma\lambda)|\theta_k^i| + \gamma M(A + b + ca_\infty).
\]
Finally, iterating again this inequality and using~\ref{assump:B1}, we obtain for all $k \in \N$ and for every $i \in \{1, \dots, N\}$
\begin{align*}
  |\theta_k^i|
  &\le (1-\gamma\lambda)^k\,|\theta_0^i| + \frac{M(A + b + ca_\infty)}{\lambda}\big(1-(1-\gamma\lambda)^k\big)\\
  &\le R_0 \vee \frac{M(A + b + ca_\infty)}{\lambda} = R_\infty,
\end{align*}
which is the claimed bound.
\end{proof}

\begin{remark}
  Here, the condition $\lambda > Mc$ plays the same role as the lower bound $\lambda > \lambda_\star$ in the bounded case: it enforces that the ridge regularization dominates the effective interaction induced by the activation growth. The additional requirement $\gamma\lambda \le 1$ is a convenient discrete-time stability condition ensuring $1 - \gamma\lambda \ge 0$, and is automatically satisfied for large $N$ under the scaling $\gamma = \alpha / N$.
\end{remark}

\subsection{Reduction to the bounded case}
\label{sec:reduction-bounded}

Lemma~\ref{lem:localization} shows that, under \ref{assump:B1}--\ref{assump:B3} and sufficiently strong regularization, the SGD iterates remain within the closed ball $\overline{B}(0,R_\infty)$ \emph{for all times}. Hence, we may work on the compact set $\overline{B}(0, R_\infty) \times K_x$, on which the activation and its gradients are uniformly bounded. To apply the arguments of~\cref{sec:proofs}, we additionally require a localized Lipschitz property for $\nabla_{\theta^i}\sigma_*$ in place of~\ref{assump:A3}:

\begin{enumerate}[label={(B\arabic*)}, font=\bfseries]
  \setcounter{enumi}{3}
  \item\label{assump:B4}
  For every $R > 0$, there exist constants $L_x(R), L_\theta(R) > 0$ such that for all
  $x,x'\in K_x$ and $\theta^i,\theta'^i\in \overline{B}(0,R) \subset \R^D$,
  \begin{align*}
    &|\nabla_{\theta^i}\sigma_*(\theta^i,x)-\nabla_{\theta^i}\sigma_*(\theta^i,x')|
    \le L_x(R)\,|x-x'|,\\
    &|\nabla_{\theta^i}\sigma_*(\theta^i,x)-\nabla_{\theta^i}\sigma_*(\theta'^i,x)|
    \le L_\theta(R)\,|\theta^i-\theta'^i|.
  \end{align*}
\end{enumerate}

Assumption~\ref{assump:B4} is mild in the present setting: for instance, it holds whenever $\sigma_*$ is $\mathcal{C}^2$ with bounded mixed second derivatives on compact sets, which includes standard smooth approximations of non-smooth activations used in practice. The arguments of~\cref{sec:proofs} then carry over with only notational changes, replacing the global constants ($B$, $M$, $L_x$, $L_\theta$) by their localized counterparts on $\overline{B}(0,R_\infty) \times K_x$.

\section*{Acknowledgements}

A.G is supported by the ANR-23-CE-40003, Conviviality, and has benefited from a government grant managed by the Agence Nationale de la Recherche under the France 2030 investment plan ANR-23-EXMA-0001. P.S is supported by the Projet I-SITE Clermont CAP 2025. 

\printbibliography

\end{document}